\documentclass[11pt,fleqn]{article}

\usepackage{a4wide}
\usepackage{latexsym}
\usepackage{amssymb, amsmath, amsthm}
\usepackage{cite}
\usepackage{url}
\usepackage{pgfplots}
\pgfplotsset{width=15cm,compat=1.8}
\usepackage{subcaption}
\usepackage{tikz}
\usetikzlibrary{plotmarks}
\usepackage{graphicx}

\theoremstyle{plain}
\newtheorem{proposition}{Proposition}
\theoremstyle{definition}
\newtheorem{example}{Example}
\newtheorem{definition}{Definition}

\title{\bf On the  Relationship between Bayesian Networks and
  Probabilistic Structural Causal Models}

\author{Peter J.F. Lucas$^1$, Eleonora Zullo$^2$ \& Fabio Stella$^2$\\
 $^1$ Faculty of EEMCS, University of Twente, Enschede, the Netherlands\\
  ICIS, Radboud University, Nijmegen, the Netherlands \\[1.5ex]
  $^2$Dipartimento di Informatica, Sistemistica e comunicazion \\
  University of Milano-Bicocca, Milan, Italy \\
  Email: peter.lucas@utwente.nl,$\{\mbox{fabio.stella,e.zullo2}\}$@unimib.it}

\date{March 28, 2026}

\begin{document}

\maketitle

\begin{abstract}

  In this paper, the relationship between probabilistic graphical
  models, in particular Bayesian networks, and causal diagrams, also
  called structural causal models, is studied.  Structural causal
  models are deterministic models, based on structural equations or
  functions, that can be provided with uncertainty by adding
  independent, unobserved random variables to the models, equipped
  with probability distributions. One question that arises is whether
  a Bayesian network that has obtained from expert knowledge or learnt
  from data can be mapped to a probabilistic structural causal model,
  and whether or not this has consequences for the network structure
  and probability distribution. We show that linear algebra and linear
  programming offer key methods for the transformation, and examine
  properties for the existence and uniqueness of solutions based on
  dimensions of the probabilistic structural model. Finally, we
  examine in what way the semantics of the models is affected by this
  transformation.
  \\[2ex]
\end{abstract}

\section{Introduction}

With the increasing interest in \emph{causality} by the probabilistic
reasoning community, several researchers, and in particularly Judea
Pearl \cite{pearl2000}, made a switch from a graphical representation
of uncertain interactions based on probability theory, towards a
deterministic representation of causal interactions with some added
external uncertain influence. For considerable time, Bayesian networks
(BNs) and related graphical models, such as maximal ancestral graphs
\cite{richardson2002}, were the main vehicle for research in
probabilistic graphical models \cite{pearl1988}. Since two decades
their role was slowly taken over by structural causal networks (SCMs)
\cite{pearl2000}.  This is a surprising shift, as originally all the
effort was aimed at developing a causal representation of uncertainty,
for example by \emph{causal} Bayesian networks \cite{fenton2012risk},
where the adjective `causal' emphasises the aim of modelling uncertain
causality. The implications of this step, although never emphasised in
a prominent way, are profound, as with a deterministic graphical
representation many of the independence assumptions developed for
probabilistic graphical models no longer hold
\cite{pearl1988,lauritzen1996}. There are also other reasons to raise
eyebrows, because if a local probability distribution is interpreted
as an uncertain causal mechanisms, possibly provided with abstraction,
why is it that causal mechanisms can be best captured by a
deterministic relationship with some added uncertainty?

Recently, structural causal models are also attracting much attention
from the deep learning community, as many mechanisms in nature and
social science can be best understood in terms of causal diagrams (e.g
\cite{Lagemann2023}), and, in addition, such mechanisms, once modelled
as causal diagrams, can be used to reliably generate synthetic data,
as demonstrated by Hollmann et al.\ \cite{hollmann2025}. Thus, structural causal
models can act as output of a machine learning process, but they offer
also potential as input to the learning process. However, also in
these situations, uncertainty will pop up. These developments
clearly show that understanding the nature of causal diagrams
is not only of interest to a small group of researchers. 

Thus, an attempt to understand the relationship between graphical
models used for uncertainty representation and those for the
representation of causal knowledge appears to be a good starting point
for research.  Part of the reasons for moving to causal networks
derives from the fact that probabilistic relationships can be
reversed, using Bayes' rule, which is a phenomenon that does not fit
well with the principle of causality. However, there are also certain
practical reason why this switch to deterministic equations is
made. It is this which we want to explore in this paper, with as as
subquestion whether this switch is always required and possible.  One
point which we want to raise is that probability distributions can
also be deterministic, thus also violating some of the basic
independence axioms of positive distributions, implying that in many
cases it is possible to use probabilistic graphical models, or
Bayesian networks in our case, to implement causality.  The most
important issue we wish to shed light on is the relationship between
structural causal network models, and in particular their
probabilistic extension, \emph{probabilistic structural causal models}
(PSCMs), and Bayesian networks. Are they equivalent, and is it always
possible to transform one into the other and vice versa?

The paper is organised as follows. We start with a brief motivating
example to illustrate the problem that is being studied in the present
paper. The example is followed by a summary of the mathematical
concepts needed to be able to answer the research questions with some
precision. The rest of Section \ref{sec.datamet} consists of a
discussion of several simple examples; it is hard to gain a feeling on
the subject of causality if one resorts to presenting only
abstractions. Next, we delve into issues concerning probability
distributions of the exogenous variables, essentially the
probabilistic extension of deterministic causal models, in Section
\ref{sec.exo}. The real novel contributions of the paper to the theory
of probabilistic structured causal models starts in Section
\ref{sec.trans}, where we show that well-known methods from linear
algebra, and in particular linear programming, are key to unravelling
the relationships between BNs and PSCMs.  A python program that
implements the theory developed in the paper, supporting
experimentation with the transformation of BNs to PSCMs, is presented
in Section \ref{sec.alg}.  The paper is concluded by a brief
discussion in Section \ref{sec.disc}.

\section{Motivating example}
\label{sec.motex}

Before going into details of the mathematical machinery that is
needed, we start by providing a very simple instance of the kind of
problems studied in this paper, concerning \emph{treatment} $T$ and its
effect $B$, here standing for \emph{blindness}.

\begin{figure}[t]
  \centering
  \subcaptionbox{\label{laba}}
  {\includegraphics[width=0.46\linewidth]{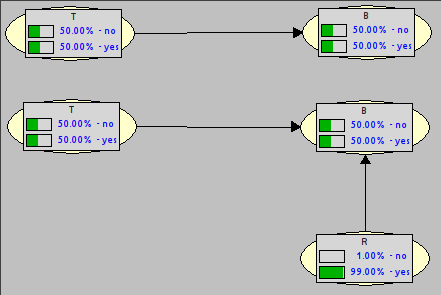}}
  \ \
  \subcaptionbox{\label{labb}}
 {\includegraphics[width=0.45\linewidth]{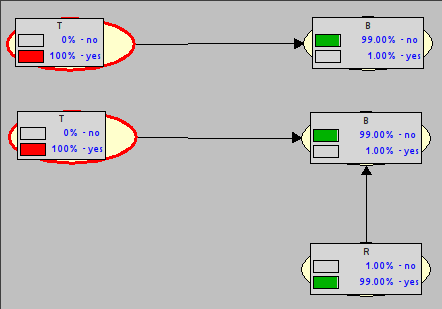}}
  \caption{Two different Bayesian network representations of the same distribution; the top one a regular Bayesian network, whereas the one at the bottom is a Bayesian network representation of a probabilistic structured causal model; (a): priors; (b) posteriors.}
  \label{one}
\end{figure}

In the Bayesian network (BN) shown in Figure \ref{one}, the
(uncertain) relationship between $T$ and $B$ is represented as the
joint probability distribution $P(T,B) = P(B \mid T) P(T)$, with
conditional distributions:
\begin{eqnarray}
  P(B = n \mid T = n) = 0.01 \\
  P(B = n \mid T = y) = 0.99
\end{eqnarray}
The following meaning is assumed: $T = y$ (resp.\ $T = n$) stands for
treatment (no treatment) and $B = y$ (resp.\ $B = n$) stands for being
blind (not being blind). Although $T$ is intended as an intervention,
it will be represented as a random variable $T$ that is uniformly
distributed. The top of Figure \ref{one} ((a) and (b)) depicts the
associated prior and instantiated Bayesian network with posteriors.

The aim is to represent this distribution by a probabilistic
structured causal model (PSCM), with a causal relationship between $T$ and $B$,
now called `endogenous' variables. Uncertainty comes from
outside, by a so-called `exogenous' variable $R$. All variables are
binary. The relationship between treatment $T$ and the presence or
absence of blindness $B$ can be represented by the function
\[f_B(T) = 1 - T = B\] The intuitive meaning is that treatment
($T = 1$) prevents a person from becoming blind ($B = 0$). As we wish
to represent this PSCM as a BN, adding uncertainty to this
deterministic relationship and function can be done by adding an extra
exogenous variable $R$ to $B$, as is shown in Figure \ref{one} ((a)
and (b)) at the bottom, and adding $R$ to the conditional probability
distributions as follows:
\begin{eqnarray*}
  P(B = n \mid T = y, R = n) = 0 \\
  P(B = n \mid T = y, R = y) = 1 \\
  P(B = n \mid T = n, R = n) = 1 \\
  P(B = n \mid T = n, R = y) = 0
\end{eqnarray*}
If we assume that the
uncertainty in $R$ is represented by the following distribution
$P(R = y) = 0.99$ and $P(R = n) = 0.01$. This will allow to use only \emph{one}
unconditional binary distribution $P(R)$ to represent the uncertainty of \emph{two}
conditional distributions in this special case, as it is possible to compute
$P(B \mid T)$ as follows using the marginalisation rule:
\begin{equation}
  P(B \mid T) = \sum_{r} P(B, R =r \mid T) = \sum_{r} P(B \mid T, R = r) P(R = r) \label{marg}
\end{equation}
As Figure \ref{one} indicates, we will get exactly the same
conditional distributions $P(B \mid T)$ for both cases, BN and BN
version of the PSCM, top and bottom.

The reason why the binary distribution $P(R)$ suffices for
representing \emph{two} conditional distributions $P(B \mid T = n)$
and $P(B \mid T = y)$ is due to the special case that
$P(B = x \mid T = n) = 1 - P(B = x \mid T = y)$ in all cases for
$x \in \{n,y\}$: not more than two probabilities can be exploited and
in this case, not more than two are needed. Clearly, this will not
hold for the general case, but it is unclear how many probabilities
are needed in general. Also unclear is how the structural causal
functions should be defined based on a given Bayesian
network. However, in order to precisely analyse the problem, some
mathematical foundations have to be provided, which will be done in
the following section.

\section{Mathematical preliminaries}
\label{sec.datamet}

We start by providing a short mathematical summary of the formal concepts that
will be utilised in the remainder of this paper.

\subsection{Bayesian networks}

Let $G = (V, A)$ be an \emph{acyclic directed graph} (DAG), with \emph{nodes}
$V = \{1,2,\ldots,n\}, n \geq 0$, and \emph{arcs}
$A \subseteq V \times V$.  The set of \emph{parents} of node $v \in V$, denoted
$\textrm{pa}(v)$, is defined as $\textrm{pa}(v) = \{u \mid (u, v) \in A\}$.
\begin{definition}[Bayesian network (BN)]
A \emph{Bayesian network}  $\mathcal{B} = (G, X, P)$ is defined as
an acyclic directed graph $G = (V, A)$, where the nodes have a 1--1 correspondence
with the \emph{random variables} $X$ which are indexed by elements in
$V$, i.e.\ $X_V = \{X_v \mid v \in V\}$. The associated (discrete)
\emph{joint probability distribution} $P$ of $\mathcal{B}$ is
factorised as follows:
\begin{eqnarray}
  P(X_V) = \prod_{v \in V} P(X_v \mid X_{\textrm{pa}(v)})
 \label{bn}
\end{eqnarray}
\end{definition}

Values of a set of variables
$X$ are often denoted by $X = x$, or sometimes simply by $x$.
An arc $(v,w) \in A$ is also
denoted by $v \rightarrow w$. The associated \emph{domain} of $X_S$,
$S \subseteq V$, is denoted by $D(X_S)$. In examples often no
distinction is made between node $v$ and variable $X_v$, in particular
when the nodes get a descriptive name.
Thus, a Bayesian network is fully specified by its set of associated
conditional probability distributions $P(X_v \mid X_{\textrm{pa}(v)})$,
one distribution for each set of values $x_{\textrm{pa}(v)}$
of the variables $X_{\textrm{pa}(v)}$, that take into account
the independence information encoded in the graph $G$.

\subsection{Structural causal models}

Structural causal models are also defined as an acyclic directed graph
as above for Bayesian networks.
\begin{definition}[Structural causal model (SCM)]
  A \emph{structural causal model}  $\mathcal{M} = (G, X, F_V)$ is
  defined as an acyclic directed graph $G = (V,A)$
with a 1--1 correspondence between nodes $V$ and variables
$X_V$, where the set of functions $F_V$ describes \emph{causal relations} and
is defined as follows:
\begin{eqnarray}
  F_V = \{f_v: D(X_{\textrm{pa}(v)}) \to D(X_v)\ \mid v \in V\}
  \label{causalrels}
\end{eqnarray}
\end{definition}

The functions $f_v$ are often represented as equations, not
necessarily linear.

So far, there is no uncertainty involved in SCMs,
as the functions $f_v$ are deterministic, and could also be
represented as a Bayesian network $\mathcal{B}$ where the functions
$f_v$ are represented as a family of deterministic conditional
probability distributions $P(X_v \mid X_{\textrm{pa}(v)})$, i.e.\ with
$P(x_v \mid x_{\textrm{pa}(v)}) \in \{0,1\}$. However, it is
possible, as already demonstrated in Section \ref{sec.motex}, to extend a causal network by adding additional, uncertain
so-called \emph{exogenous} variables $U_W$ (and exogenous nodes
$W = \{n+1,\ldots,m\}$) to an SCM, where $X_V$ are now called
\emph{endogenous} variables, yielding a \emph{probabilistic structural
  causal model} (PSCM) $\mathcal{M} = (G,X_, U,P')$ \cite{zaffalon2020},
with $G = (V \cup W, A)$, $A \subseteq (V \cup W) \times V$, with
only \emph{one} arc $(w,v) \in A$, for each $w \in W$ and $v \in V$, and associated joint probability distribution
\begin{eqnarray}
  P'(X_V,U_W) = \prod_{v \in V} P'(X_v \mid X_{\textrm{pa}(v)})\prod_{w \in W}P'(U_w) .
\end{eqnarray}
Note that the variables $U_w$, $w \in W$, are assumed to be
independent, i.e.
\[P'(U_W) = \prod_{w \in W} P'(U_w)\]
and that parents
$\textrm{pa}(v)$ include endogenous and exogenous nodes. In addition,
the conditional probability distributions
$P'(X_v \mid X_{\textrm{pa}(v)})$ are deterministic and follow the
specification of the associated functions $f_v$ of the SCM. Thus, the
only uncertainty represented in a PSCM is in its associated exogenous
variables $U_w$. We follow Zaffalon et al. \cite{zaffalon2020} in
employing a special notation for parents of endogenous variables with
the exogenous variable removed:
$X_{\underline{\textrm{pa}}(v)} = X_{\textrm{pa}(v)} \backslash U_W$,
for each $v \in V$.

Note that, given the joint probability distribution $P'(X_V,U_W)$,
one can compute the probability distribution $P'(X_V)$ simply
by maginalisation:
\[P'(X_V) = \sum_{u_W \in D(U_W)} P'(X_V,U_W = u_W)\]
and the resulting $P'(X_V)$ will in general be nondeterministic. As
we also have that
\[P'(X_V) = \prod_{v \in V} P'(X_v \mid  X_{\underline{\textrm{pa}}(v)})\]
where $P'(X_v \mid X_{\underline{\textrm{pa}}(v)})$ will not be deterministic
in most cases,
the question is whether it is
possible to recover the conditional probability distributions
$P(X_v \mid X_{\underline{\textrm{pa}}(v)})$ based on the given
$P'(U_W)$ and (deterministic) $P'(X_v \mid X_{\textrm{pa}(v)})$,
such that $P(X_v) = P'(X_v)$.

The solution provided by Zaffalon et al. \cite{zaffalon2020} makes use
of the relationship between a given endogenous variable $X_v$ and associated
exogenous variable $U_w$, with $(w,v) \in A$; the function $f_v$
maps $u_w$ and $x_{\underline{\textrm{pa}}(v)}$ to $x_v$, as follows:
\[f_v : D(U_w) \times D(X_{\underline{\textrm{pa}}(v)}) \to D(X_v)\]
i.e.\ $f_v(u_w,x_{\underline{\textrm{pa}}(v)}) = f_v(x_{{\textrm{pa}}(v)}) = x_v \in D(X_v)$.  As
each tuple of values $(u_w,x_{\underline{\textrm{pa}}(v)})$ is mapped
to a value $x_v$ of $X_v$, each values of $D(X_v)$ covered, the
function $f_v$ is assumed to be surjective, and not
necessarily injective. It may be that
more than one tuple $(u_w,x_{\underline{\textrm{pa}}(v)})$ is mapped
to the same $x_v$, but each element of $D(X_v)$ has at least one
associated tuple from
$D(U_w) \times D(X_{\underline{\textrm{pa}}(v)})$.  The surjectivity
of $f_v$ implies that there exist a right inverse function $f^{-1}_v$,
when restricted to a fixed $\underline{\textrm{pa}}(v)$,
an inverse partial function is obtained:
\[f^{-1}_{v|\underline{\textrm{pa}}(v)}: D(X_v) \to \wp(D(U_w))\]
i.e.\
$f^{-1}_{v|\underline{\textrm{pa}}(v)}(x_v)$
may yield a set of values from $D(U_w)$, which is a consequence of
the surjectivity of $f_v$. If $f_v$ is also injective, then the
range of $f^{-1}_{v|\underline{\textrm{pa}}(v)}$ will consist of
singleton sets, which also may be represented by their elements. The
requirement of surjectivity of $f_{v|\textrm{pa}(v)}$ comes from the requirement that
every element of $D(U_w)$ needs to map to at least one
element of $D(X_v)$. This will allow computing the elements
of $P'(x_v \mid x_{\underline{\textrm{pa}}(v)})$ as follows:
\begin{eqnarray}
  P'(x_v \mid x_{\underline{\textrm{pa}}(v)}) = \sum_{u_w \in f^{-1}_{v|\underline{\textrm{pa}}(v)}(x_v)}P'(U_w = u_w)
  \label{inverse}
\end{eqnarray}
A consequence of this definition is that $P'(x_v \mid x_{\underline{\textrm{pa}}(v)})$ can be approximated by a uniform distribution $P'(U_w)$, similarly
to a Riemann sum (take as many value of $U_w$, and sum the associated probabilities, as needed to obtain $P'(x_v \mid x_{\underline{\textrm{pa}}(v)})$), for appropriate definitions of $f_v$. However, this approach would make $P'(U_w)$, and thus
also $P'(x_v \mid x_{\underline{\textrm{pa}}(v)})$, meaningless. This implies that the meaning of $P'(U_w)$ would become unimportant, which clearly
is unacceptable, as what would be the value of a meaningless
causal network?

Employing the right inverse of the function $f_v$ is not the
only way to compute $P'(X_v \mid x_{\underline{\textrm{pa}}(v)})$,
as given $P'(U_w), w \in W$, we can also determine $P'(X_v \mid x_{\underline{\textrm{pa}}(v)})$
for any value of $x_{\underline{\textrm{pa}}(v)} \in D(X_{\underline{\textrm{pa}}(v)})$ using marginalisation as follows:
\begin{eqnarray}
  P'(X_v \mid x_{\underline{\textrm{pa}}(v)}) = \sum_{u_w \in D(U_w)}
  P'(X_v \mid x_{\underline{\textrm{pa}}(v)}, U_w = u_w) P'(U_w = u) \label{marg0}
\end{eqnarray}
which offers a
different approach, with essentially the
same relationship between $P'(U_w)$ and $P'(X_v \mid x_{\underline{\textrm{pa}}(v)})$. Using either of the two approaches, the issue is whether or not the
BN $\mathcal{B}$ can be transformed to a PSCM $\mathcal{M}$, and back,
yielding the same BN, and vice versa, or
\[P(X_v \mid X_{\textrm{pa}(v)}) = P'(X_v \mid X_{\underline{\textrm{pa}}(v)})\] for each $v \in V$, and thus
whether $P(X_V) = P'(X_V)$.
\begin{example}
\label{ex1}
  \begin{figure}
    \centering
 \subcaptionbox{\label{a}}
 {\includegraphics[width=0.42\linewidth]{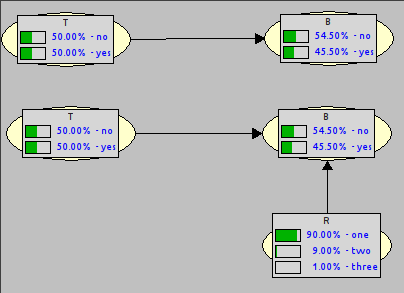}} \ \
 \subcaptionbox{\label{b}}
 {\includegraphics[width=0.4\linewidth]{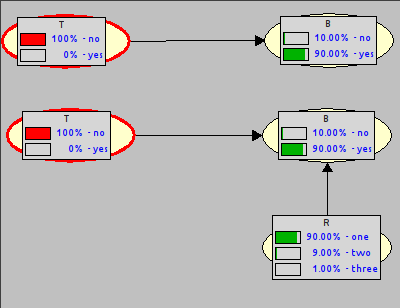}}
 \caption{Two different Bayesian network representations with two conditional
   distributions and more values for $R$, before and after entering
   an intervention $T$; (a): priors; (b): posteriors.}
  \label{two}
\end{figure}
The conditional probabilities of the top BN of Figure \ref{two}
are defined as follows:
\begin{eqnarray*}
  P(B = n \mid T = n) = 0.1 \\
  P(B = n \mid T = y) = 0.99
\end{eqnarray*}
Thus, in contrast to the motivating example of Section \ref{sec.motex},
these two probabilities no longer add up to 1, implying
that the two probability distributions can no longer be represented
by a single unconditional binary distribution, $P(R)$ in this case.
With a probability distribution of $R$ consisting of three different
values:
\begin{eqnarray*}
  P(R = \textrm{one}) = 0.9 \\
  P(R = \textrm{two}) = 0.09 \\
  P(R = \textrm{three}) = 0.01
\end{eqnarray*}
where intentionally $0.9 +0.09 = 0.99$ and $0.09 + 0.01 = 0.1$,
and the associated family of deterministic probability distributions:
\begin{eqnarray*}
  P(B = n \mid T = y, R =  \textrm{one}) = 1 \\
  P(B = n \mid T = y, R =  \textrm{two}) = 1 \\
  P(B = n \mid T = y, R = \textrm{three}) = 0 \\
  P(B = n \mid T = n, R =  \textrm{one}) = 0 \\
  P(B = n \mid T = n, R =  \textrm{two}) = 1 \\
  P(B = n \mid T = n, R = \textrm{three}) = 1
\end{eqnarray*}
using equation (\ref{marg}) it is possible to select the first two
probabilities of $P(R)$ to obtain $P(B = n \mid T= y)$ and the last two
probabilities to obtain $P(B = n \mid T = n)$, as shown
on the right-hand side of Figure \ref{two}. As a consequence,
the two BNs in Figure \ref{two} are equivalent, and after
marginalisation out $R$, identical.

Corresponding to the deterministic conditional distributions
above, is the following definition of $f_B$:
\[f_B(T,R) = (T = 0)\cdot(R < 2) + T\cdot(R > 2) = B\] e.g.\
$f_B(1,2) = 0 = B$ (not blind, i.e.\ $B = n$).  The computations by
using the marginalisation rule (equation (\ref{marg})) are very
similar to the use of the inverse partial function
$f^{-1}_{v|\underline{\textrm{pa}}(v)}$ with two values (0.9 and 0.09)
from $D(R)$ associated to $B = n \mid T=y$ and two other values (0.09
and 0.01) associated to $B = n \mid T = n$.  Note that the value 0.09
is used twice in the computations, one time to obtain
$P(B = n \mid T = y)$ and another time to obtain
$P(B = n \mid T = n)$. The inverse of the partial function
$f_{B|T=1}$, denoted $f_{B|T=y}$, i.e.\ $f^{-1}_{B|T=y}$, would give
the result
$f^{-1}_{B|T=y}(0) = \{1,2\} = \{\textrm{one},\textrm{two}\}$. The
alternative computation based on Equation (\ref{inverse}) would yield
\[\sum_{r \in f^{-1}_{B|T=y}(0)}P(R = r) = 0.9 +0.09 = 0.99 = P(B = n \mid T = y)\]
\end{example}
Thus, by decomposing the conditional probabilities
$P(X_v \mid X_{\underline{\textrm{pa}(v)}})$ in a proper way giving
rise to $P(U_w)$, and the right definition of the deterministic
conditional probabilities $P(x_v \mid x_{\underline{\textrm{pa}(v)}})$,
it is possible to recover a Bayesian network in this way.
The reader
has probably already noticed that the definition of $P(R)$
in Example \ref{ex1} is not the only possibility to recover
the Bayesian network, which will be demonstrated next.
\begin{example}
  \label{ex2}
  Reconsider Example \ref{ex1} with the given family of conditional
  probability distributions $P(B \mid T)$. As we wish to
  recover the probabilities 0.99 and 0.1 from $P(R)$, a simple other
  definition is obtained by realising that 0.99 is obtained
  by the sum $0.89 + 0.1$ This implies that the following
\begin{eqnarray*}
  P(R = \textrm{one}) = 0.89 \\
  P(R = \textrm{two}) = 0.1 \\
  P(R = \textrm{three}) = 0.01
\end{eqnarray*}
may also be a suitable definition to recover the Bayesian network from
a PSCM. However, when accepting this definition, the deterministic probability
distributions also needs to be modified, as follows:
\begin{eqnarray*}
  P(B = n \mid T = y, R =  \textrm{one}) = 1 \\
  P(B = n \mid T = y, R =  \textrm{two}) = 1 \\
  P(B = n \mid T = y, R = \textrm{three}) = 0 \\
  P(B = n \mid T = n, R =  \textrm{one}) = 0 \\
  P(B = n \mid T = n, R =  \textrm{two}) = 1 \\
  P(B = n \mid T = n, R = \textrm{three}) = 0
\end{eqnarray*}
This leads for example to:
\[P(B = n \mid T = y) = \sum_r P(B = n \mid T = y, R = r)P(R = r)
  = 1 \cdot 0.89 + 1 \cdot 0.1 + 0 \cdot 0.01 = 0.99\]
After redefining the functions $f_{B|T}$, the same result
would be obtained using Equation (\ref{inverse}).
\end{example}

\section{Relationship between exogenous and endogenous variables}
\label{sec.exo}
\subsection{How many values do we need?}

Based on the discussion above, an obvious question is how Equations
(\ref{inverse}) and (\ref{marg0}) would work if all variables in a
PSCM, including the exogenous variables $U_w$, $w \in W$, were assumed
to be binary, or whether a Bernoulli $P(U_w)$ would be able to capture
the uncertainty in the original BN (which may or may not have been
learnt from data).  This is illustrated by means of the following
example.
\begin{example}
\label{ex4}
Returning to the motivating example in Section \ref{sec.motex} and Figure
\ref{one}, application of Equation (\ref{marg0}) is easy in this case,
as the set $\wp(D(U_w)) = \wp(D(R))$ only contains two singleton sets
$\{y\},\{n\}$ with the associated probabilities copied once from
$P(B = n \mid T = y)$ and once from $P(B = n \mid T = n)$. Equation
(\ref{inverse}) will give the same result.

For both conditional probability distributions we have the special case that
$P(B = x \mid T = n) = 1 - P(B = x \mid T = y)$ for
$x \in \{n,y\}$. This special condition is needed, as $P(R)$ only
gives two probabilities, namely $P(R = y)$ and $P(R = n)$ with
$P(R = y) = 1 - P(R = n)$. As $P(R)$ is used to render the conditional
probabilities $P(B \mid T, R)$ uncertain, not more than two
probabilities can be exploited, whereas four are needed in general
when the two conditional probabilities do not add to 1.
\end{example}
The example given above shows that the approach provided by Zaffalon
et al.\ \cite{zaffalon2020} requires knowledge about the distributions
of both endogenous and exogenous variables before one can start using
Equations (\ref{inverse}) and (\ref{marg0}).  Although Example
\ref{ex1} indicates that the method will not work for arbitrary
conditional distributions and a binary distribution, the approach
supports defining non-binary distributions of the exogenous
variable. One needs first to fix a domain $D(U_w)$ with sufficient
number of values.  By constituting subsets of values of $D(U_w)$,
probabilities can be combined in such way, by summing over the
probabilities associated with a subset values, that the conditional
probabilities $P(x_v \mid x_{\underline{\textrm{pa}}(v)})$ can be
recovered. This process is discussed in the following section.

\subsection{Formulation of the general case}
\label{bntocn}

\begin{definition}[Partition of domain]
  Let $X$ be a variable, with domain $D(X)$ and
  discrete probability distribution $P(X)$, then $\Pi(X)$ denotes
  a \emph{partition of domain} $D(X)$ consisting of subsets  $S \subseteq D(X)$,
  such that $\bigcup_{S \in \Pi(X)} S = D(X)$ --- a partition is exhaustive ---, and each $S, S' \in \Pi(X)$,
  $S \neq S'$, are disjoint.
\end{definition}
This concept of partition of a domain is needed for extracting probabilistic
information from the exogenous variables, and will be used to
obtain matching probabilistic information for the endogenous variables.
\begin{definition}[Probability assignment]
  Let $U$ be an exogenous variable, $P(U)$ its associated
  discrete probability distribution, and $\Pi(U) = \{S_1,\ldots,S_m\}$ a partition
  of $D(U)$. Let $P(x_v \mid x_{\underline{\textrm{pa}}(v)})$ be the conditional
  probability with $U \in X_{{\textrm{pa}}(v)}$ and
  $P(x_v \mid X_{\underline{\textrm{pa}}(v)}) = \{P(x_v \mid x^i_{\underline{\textrm{pa}}(v)}) \mid i=1,\ldots,m\}$.
  A \emph{probability assignment} is then defined as follows:
  \[P(x_v \mid x^i_{\underline{\textrm{pa}}(v)}) = \sum_{u \in S_i} P(U = u).\]
\end{definition}
In Section \ref{sec.motex} we have discussed an example, where the
partition of $D(R)$ consisted of two singleton sets. However, any partition
can be used, not just singleton sets. In any case, the following (obvious) proposition holds:
\begin{proposition}\label{propassign}
  Let $P(x_v \mid X_{\underline{\textrm{pa}}(v)})$ be the result of a probability assignment,
  then it holds that:
  \[\sum_{i=1}^m P(x_v \mid x^i_{\underline{\textrm{pa}}(v)}) = 1\]
\end{proposition}
\begin{proof}
  As $\sum_{u \in D(U)} P(U = u) = 1$, it follows that $\sum_{S \in \Pi(U)}\sum_{u \in S}P(U = u) = 1$.
\end{proof}
The definition of a probability assignment also supports using different permutations of
probabilities, which essentially yield different probability distributions $P(U)$,
but by putting the same combinations of probabilities into a partition, the same family
of conditional probability distributions can be obtained.

However, as was demonstrated by Examples \ref{ex1} and \ref{ex2}, it is also
possible to form clusters of probabilities from $P(U)$ to define
conditional probabilities, violating the requirement of a partition
that its members are disjoint and exhaustive.
\begin{example}
  The probability distributions $P(R)$ defined in Examples \ref{ex1}
  and \ref{ex2} yield sets of clusters
  $\{\{0.9, 0.09\}, \{0.09, 0.01\}\}$ and
  $\{\{0.89, 0.1\}, \{0.1\}\}$, resp., which do not constitute a
  partition (since the first and second example have overlapping sets,
  where in addition for the second example the sets are not exhaustive)
  and in both cases Proposition \ref{propassign} does not apply
  because the partition requirement fails to hold. 
\end{example}
So far it is assumed that there is no knowledge
represented in a given BN to build the causal relationships
(\ref{causalrels}):
\begin{eqnarray*}
F_V = \{f_v: D(X_{\textrm{pa}(v)}) \to D(X_v)\ \mid v \in V\}
\end{eqnarray*}
of a PSCM based on a given BN. In reality, however, some relationships
between parent variables and a child variable are so strong (reflected
by a high probability) that one may expect that they carry over to the
causal functional relationship.  Unfortunately, the exogenous
variables throw a spanner in the works. As its associated probability
distribution models the uncertainty in the causal relationship, as we have
seen above, usually several probabilities from this distribution are
needed to (re)build the conditional probability distributions of a
BN. By reordering a distribution of an exogenous variable (for example
from high to low probability), one can confidently say that the strong
relationships must come from the high end of the distribution. However,
not much more can be concluded from this observation.

\subsection{Offers adding arcs an alternative method?}

As has been demonstrated in the previous section, the probability
distribution of the exogenous variables $P(U_W)$ can be used to add
uncertainty to the functional relationship between causes and the
effects expressed by the functions $f_v$, but without having a clear
meaning. It was also shown that different probability distributions
$P(U_W)$ can be used to transform a BN to a PSCM and vise versa. Thus,
as a preliminary conclusion it can be stated that a relationship
between BNs and PSCMs exists, although not necessarily unique.

This conclusion raises the question whether there are alternative ways
to transform a BN into a PSCM.
\begin{example}
  Reconsider the BN from Example \ref{ex1}.  One potential alternative
  solution is to add an arc from vertex $T$ to $R$.  This is done in
  Figure \ref{twodist}.
\begin{figure}[h]
  \centering
\subcaptionbox{\label{labaa}}
{\includegraphics[width=0.4\linewidth]{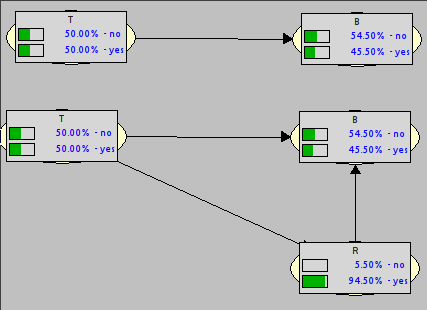}}
\ \
\subcaptionbox{\label{labbb}}
 {\includegraphics[width=0.4\linewidth]{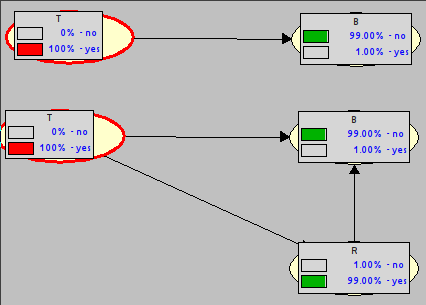}}
  \caption{Two different Bayesian network representations of the same distribution with one extra arc from $T$ to $R$ with (a): priors; (b) posteriors.}
  \label{twodist}
\end{figure}

To be able to represent the same distribution using a deterministic
causal network, i.e.\ a PSCM, we draw an arc from $T$ to $R$ and
define $P(R \mid T)$ as follows:
\begin{eqnarray}
  P(R = n \mid T = n) = 0.1 \\
  P(R = n \mid T = y) = 0.01
\end{eqnarray}
Using
\begin{equation}
  P(B \mid T) = \sum_r P(B, R = r\mid T) = \sum_r P(B \mid T, R = r) P(R = r \mid T) \label{marg2}
\end{equation}
we are able to compute the distribution after conditioning on values
for $T$, resulting in the same a posterior
distribution, as shown in Figure \ref{twodist}. The conditional
probability distributions $P(B \mid T,R)$ are the same as
for Example \ref{ex1}:
\begin{eqnarray*}
  P(B = n \mid T = y, R = n) = 0 \\
  P(B = n \mid T = y, R = y) = 1 \\
  P(B = n \mid T = n, R = n) = 1 \\
  P(B = n \mid T = n, R = y) = 0
\end{eqnarray*}
\end{example}

This shows that by adding random variables to a Bayesian network after
modifying a set of uncertain conditional distributions to
deterministic conditional distributions, it is possible to represent
the same conditional distribution, $P(B \mid T)$ in our example, by
adding an exogenous variable and adding an arc from the conditioning
variable to the new, exogenous variable. Hence, an alternative method
exists, but it is still very similar to the original method
introduced earlier, and actually adds extra complexity to the method
without offering a clear advantage. Because of this, the method will
not be pursued any further.

\section{Transforming a BN to a PSCM}
\label{sec.trans}
The preceding section was primary meant to provide a basic
understanding of the potential relationships between BNs and PSCMs. In
the present section we develop a deeper analysis of the relationship,
and in particular we formulate an answer to the following
questions: (1) is there a procedure for turning a Bayesian network
into a PSCM, and (2) would that yield a causal network that in some
sense is equivalent to the original BN? The way these questions are
approached is by turning to linear algebra, and in particular linear
programming. It will become clear in the following why this is a
suitable approach.

Assuming that we have a family of conditional deterministic
probability distributions
$P'(X_v \mid x_{\underline{\textrm{pa}}(v)}, U_w = u_w)$, and a family
of probability distributions
$P'(x_v \mid x_{\underline{\textrm{pa}}(v)})$, we will solve for
$P'(U_w = u_w)$. This process is elaborated in the subsequent
subsections, where first some common linear algebra concepts
are reiterated.

\subsection{Linear algebra notation and concepts}
\label{sec.linalg}

Work on algebraic formulation of probability theory started with
George Boole \cite{boole1854}, which was later picked up again by
Castillo et al.\ \cite{castillo1997}. In the present context,
linear algebra offers a good starting point for analysis.

Only a minimal amount of linear-algebra notation and concepts will be
repeated here, based on the book by Strang \cite{strang1980}, 
to avoid misunderstanding in the notation used in the following and
to refresh the memory of some readers.

A \emph{vector space} $V$ consists of vectors that are closed under
addition and scalar multiplication; the zero vector is always part of
a vector space. If a vector space $V$ consists of all linear
combinations of vectors $v_1,v_2,\ldots,v_n$, i.e., $v \in V$ if
$v = \sum_{i=1}^n c_i v_i$, with $c_i \neq 0$, $i = 1,\ldots,n$,
$n \geq 0$, then it is said that these vectors \emph{span} the vector
space, denoted $\mathcal{L}(v_1,\ldots,v_n)$. A \emph{basis} for a
vector space $V$ is a set of vectors $\{v_1,\ldots,v_n\}$ that is
\emph{linearly independent}, i.e.\ $\sum_{i=1}^n c_i v_i \neq 0$
unless $c_1 = c_2 = \cdots = c_n = 0$, and that span the space.

We use upper-case letters, such as $A$, for $m \times n$
\emph{matrices}, with $m \times n$ its \emph{dimension}, lower-cases
letters for \emph{column vectors}, such as $x$, whereas \emph{row
  vectors} are denoted by the \emph{transpose} of the column-vector:
$x^T$. In addition, use will be made of the \emph{column space} of a
matrix $A$, denoted $\mathcal{R}(A)$, $\mathcal{R}$ stands for
\emph{range}, where $A$ is interpreted as a linear transformation
$T_A: \mathbb{R}^n \to \mathbb{R}^m$. A transformation $T_A(x) = b$,
with $T_A$ a linear transformation based on $A$, thus corresponds to a
set of linear equations $A x = b$, with $x \in \mathbb{R}^n$ and
$b \in \mathbb{R}^m$. The maximal number of elements of the subset of
columns of a matrix $A$ that span $\mathcal{R}(A)$ is called the
\emph{rank} of $A$. It happens that the rank of the column space of
$A$ is equal to the rank of the row space of $A$, i.e.,
$\textrm{dim} \; \mathcal{R}(A) = \textrm{dim} \; \mathcal{R}(A^T) =
r(A)$, despite the fact that the column space $\mathcal{R}(A)$ and row
space $\mathcal{R}(A^T)$ are usually different.  The existence of a
solution $x$ of the equation $A x = b$ means that $b$ is in the column
space $\mathcal{R}(A)$ spanned by the column vectors of $A$.

\subsection{Linear algebra formulation}

Equation (\ref{marg0}) is a linear equation, with
$P'(X_v \mid x_{\underline{\textrm{pa}}(v)}, U_w = u_w) \in \{0,1\}$,
whereas $P'(U_w = u_w)$, for brevity represented as $u_w$, can be
taken as a vector:
\begin{eqnarray*}
  u = \left[\begin{array}{c}
          P'(U_w = u_1) \\
          P'(U_w = u_2) \\
          \vdots \\
          P'(U_w = u_{n})
        \end{array}
  \right] =
  \left[\begin{array}{c}
          u_1 \\
          u_2 \\
          \vdots \\
          u_{n}
        \end{array}
  \right]
\end{eqnarray*}
The Boolean $P'(X_v \mid X_{\underline{\textrm{pa}}(v)}=  x_i, U_w = u_w)$ can be
represented as an element of a matrix $A$, with
\[a_{ij} = P'(X_v = x_v \mid X_{\underline{\textrm{pa}}(v)}= x_i, U_w = u_j)\]
Combined,  we get the following inhomogeneous linear equation:
\[A u = \left[\begin{array}{cccc}
                    a_{1,1} \!& \!a_{1,2} \!& \!\cdots \!& a_{1,n} \\
                    a_{2,1} \!& \!a_{2,2} \!& \!\cdots \!& a_{2,n} \\
                    \vdots \!&\!\vdots \!&\! \vdots \!&\! \vdots \\
                    a_{m-1,1} \!&\! a_{m-1,2} \!&\! \cdots \!&\! a_{m-1,n} \\
                    1  \!&\!   1  \!&\!  \cdots \!  &\! 1
                  \end{array}\right]
                \left[\begin{array}{c}
          u_1 \\
          u_2 \\
          \vdots \\
          u_{n}
        \end{array}
      \right]
      =b = \left[\begin{array}{c}
                   b_1 \\
                   b_2 \\
                   \vdots \\
                   b_{m-1} \\
                   b_{m}
                 \end{array}\right]
=
               \left[\begin{array}{c}
                   P'(x_v \mid x_{\underline{\textrm{pa}}(v)} = x_1) \\
                   P'(x_v \mid x_{\underline{\textrm{pa}}(v)} = x_2) \\
                   \vdots \\
                       P'(x_v \mid x_{\underline{\textrm{pa}}(v)} = x_{m-1}) \\
                       1
                   \end{array}\right]                   
\]
The last row of $A$, together with $b_m$, which consists of 1s only,
follows from the fact that $u$ is a probability vector:
$\sum_{i =1}^n u_i = 1$. From now on, it is assumed that vector $b$
contains a last, $m$th element equal to 1.

In solving the linear equation $Au = b$, with $m\times n$ matrix $A$,
and both $A$ and $b$ known, there is \emph{at least one solution} $u$
if and only if the number of rows $m$ is equal to the rank of the
matrix $A$, $r(A)$, i.e. $m = r(A)$ \cite{strang1980}. Subsequently,
it is assumed that
\[P(X_v \mid x_{\textrm{pa}(v)}) = P'(X_v \mid x_{\textrm{pa}(v)})\]
for each $x_{\textrm{pa}(v)}$ and $v \in V$, as this will give
the required relationship between PSCMs and BNs.

Now, in contrast to what is common in solving linear equations, in the
present case both $A$ and $u$ are \emph{unknown}, and only $b$ is
known.  That clearly indicates that it may not be easy to come up with
a solution. What helps, however, is that we know that the matrix $A$ is Boolean
(or binary), $a_{ij} \in \{0,1\}$, and that we need at least $r(A)$ linear
independent column vectors $a$ from $A$ that span the
column space of $A$. As $A$ is unknown, we could generate those linear independent column
vectors based on a fixed dimension and potentially based on
the original conditional probabilities $P(x_v \mid x_{\textrm{pa}(v)})$,
and solve for $u$. This generation process is clearly
the most straightforward way to tackle this problem.

\subsection{Towards linear programming}

The formulation above is incomplete, as a solution $u$ may include
negative elements and elements bigger than 1, even though the sum of
the components of $u$ is guaranteed to sum to 1. We can solve this
either by linear programming (LP for short), where most of the
constraints are equality constraints, or by solving the set of linear
equations by standard linear algebra, and if a solution is found,
checking whether it satisfies the constraints of a probability
distribution. The only advantage of linear programming in this case is
that its feasible solutions are guaranteed to satisfy the inequality
constraints as well. As the simplest solution is that of solving a set
of linear equations, we discuss that particular approach first.

In order to handle the inequalities with respect to $u$ mentioned above, it is necessary to include extra constraints to the linear equation
$A u = b$. We need to add the constraints that $0 \leq u_i \leq 1$, for each
$i = 1,\ldots,n$. This is done by adding a kind of \emph{slack variables} $w_i$ to the
probability vector $u$ expressing the situation that $u_i \leq 1$, for $i=1,\ldots,n$, i.e.\
\[w_i = 1 - u_i \Leftrightarrow u_i + w_i = 1\]
with $u_i, w_i \geq 0$, for $i= 1,\ldots,n$. The extended linear
equation $A' x = b'$ that is obtained, now has one extra
constraint that $x = [u \;\;w]^T \geq 0$:
\[A' u' = \left[\begin{array}{cc}
                   A & O \\
                   I & I
                  \end{array}\right]
                \left[\begin{array}{c}
                        u \\
                        w
                      \end{array}\right] = \left[\begin{array}{c}b \\ e\end{array}\right] = b'
\]
with $I$ the $n \times n$ identity matrix, $e = [1 \; 1\; \cdots \; 1]^T$ the
column vector with 1s, and $O$ the $m \times n$
null matrix (with zeros). Recall that the last row of $A$ and also of $O$
enforce that $u$ is a probability vector.

Next we demonstrate how to use this approach to transform
a simple Bayesian network, before moving on to other issues
such as the number of elements of probability
vector needed to be able to perform the transformation.
\begin{example}
\label{ex3}
Reconsider Example \ref{ex1} and Figure \ref{two}.
Formulated as a linear equation with constraints, the probabilities
presented in the example can be represented as follows:
\[A'\left[\begin{array}{c}u \\ w\end{array}\right] = \left[\begin{array}{cc}
                   A & O \\
                   I & I
        \end{array}\right]\left[\begin{array}{c}u \\ w\end{array}\right]
      = \left[\begin{array}{cccrrr}
          1 & 1 & 0 & 0 & 0 & 0 \\
          0 & 1 & 1 & 0  & 0 & 0 \\
          1 & 1 & 1 & 0  & 0 & 0 \\                
          1 & 0 & 0 & 1  & 0 & 0 \\
          0 & 1 & 0 & 0  & 1 & 0 \\
          0 & 0 & 1 & 0  & 0 & 1
          \end{array}\right]
        \left[\begin{array}{c}u \\ w\end{array}\right] =
        b' = \left[\begin{array}{c}0.99 \\ 0.1 \\ 1 \\ 1\\ 1 \\ 1\end{array}\right]
\]
with solution vector $x = [u \; w]^T \geq 0$. A unique solution is
$[u \; w]^T = [0.9 \; 0.09 \; 0.01 \; 0.1 \; 0.91 \; 0.99]^T$, where
$u$ are variables and $w$ the constraint (slack) variables. In this
case, we have that the rank of $A'$, $r(A') = 6$. As $A'$ is a
$6 \times 6$ square matrix, the matrix is invertible if it is
non-singular (the right inverse is equal to the left inverse), and
that a solution $[u \; w]^T$ exists for any vector $b'$.  Observe that
the vectors $u$ and $w$ satisfy the specified constraint that
$u_i + w_i = 1$, $i \leq i \leq 3$.  In this case, the corresponding
probability distribution of $R$ consisting of three different values
is:
\begin{eqnarray*}
  P(R = \textrm{one}) = 0.9 \\
  P(R = \textrm{two}) = 0.09 \\
  P(R = \textrm{three}) = 0.01
\end{eqnarray*}
This unique solution is also a feasible solution obtained by
linear programming.

Next, consider the matrix $A''$:
\[A'' = \left[\begin{array}{cccrrr}
          0 & 1 & 0 & 0 & 0 & 0 \\
          0 & 1 & 1 & 0  & 0 & 0 \\
          1 & 1 & 1 & 0  & 0 & 0 \\                
          1 & 0 & 0 & 1  & 0 & 0 \\
          0 & 1 & 0 & 0  & 1 & 0 \\
          0 & 0 & 1 & 0  & 0 & 1             
        \end{array}\right]
\]
The only difference with $A'$ is in component $a_{11}$. In this case,
solving the corresponding set of linear equations yields:
$[u \; w]^T = [0.9 \; 0.99 \; \mbox{$-0.89$} \; 0.1 \; 0.01 \; 1.89]^T$, where
the probability constraints are being violated. Linear programming will
simply notice that a feasible solution does not exist.
\end{example}
The example above illustrates that the matrix $A$ (the submatrix of $A'$) does
nothing else than selecting elements from the vector $u$ such that
the sum of those elements equals an element $b'_i$ (with $i=1,2$ in the example).

The linear programming approach offers something extra in comparison
to the linear algebra approach: its solution $x$ optimises
(minimises) a linear function $c^T x$, called the \emph{objective
  function}. This function can be used to formulate extra constraints
on the solution vector $x$. However, it is not immediately clear in
what way to use the objective function as part of
retrieving the probability distribution of the original
Bayesian network.

\subsection{Existence and uniqueness of solutions of the system of \\ linear inequalities}
\label{sec.exist}

Taking the linear programming formulation as a start, i.e.
\[A'\left[\begin{array}{c}u \\ w\end{array}\right] = \left[\begin{array}{cc}
                   A & O \\
                   I & I
           \end{array}\right]\left[\begin{array}{c}u \\ w\end{array}\right]
         = b'\]
with $[u \; w]^T \geq 0$,
we simplify the mathematical formulation for ease of exposition, by
rewriting the above equation to
\[A' x = b'\] with $x = [u \; w]^T \geq 0$, and $\min_{x} c^T x$, the
objective function that is being minimised.  Thus, from now on, $A'$ is taken as an
$m \times n$ matrix, $x$ as an $n$-dimensional column vector, and $b'$
as an $m$-dimensional column vector. Whether $A'$ is square ($m = n$),
or rectangular with either $m > n$ or $m < n$, is fully determined by
how many components of $x$ (or actually subvector $u$) are summed
together to obtain components of $b'$. The general rule is that the
larger $n$ is, the easier it becomes to reconstruct the original
causal Bayesian network as there will be more freedom in selecting
probabilities. In other words, how many probabilities $P(U_w = u_w)$
are selected by the deterministic conditional probabilities
$P'(X_v = x_v \mid x_{\underline{\textrm{pa}}(v)}, U_w = u_w)$, to
obtain $P'(X_v = x_v \mid x_{\underline{\textrm{pa}}(v)})$, i.e.,
$b'$, is governed by the matrix $A'$.

The rank and the $m \times n$ dimension of the matrix $A'$ allow distinguishing
various cases, where again we follow the book by Strang \cite{strang1980}:
\begin{itemize}
\item ($m = n$) The problem is to find a Boolean submatrix $A$ of
  $A'$, that, when embedded in $A'$, is invertible, yielding the
  elements of $[u\; w]^T$ from the multiplication ${A'}^{-1} b'$. If
  $A'$ is \emph{invertible}, and thus a bijective linear transformation,
  then we always have a unique solution;
  $x = {A'}^{-1} b'$, and the objective function $c^T x$ used in linear
  programming simply yields a value that is also unique. The only use
  of linear programming here, in contrast to the standard linear algebra of
  equations, is to ensure that $u$ is a probability vector, and
  to disregard solutions that fail to meet that requirement.
  The rank of the matrix $A'$ is equal to $r(A')= r = m = n$, and we have
  only basic variables and no free variables. However, if $r(A') < n$,
  then there are free variables. This situation is similar to
  what is discussed for the cases with $m \neq n$ below.
\item ($m > n$) The system $A' x = b'$ has \emph{at most one} solution
  if and only if the $n$ columns of $A'$ are linearly independent,
  i.e., $r(A') = r = n$. In this case, $A'$ is an injective linear
  transformation, and there exists an $n \times m$ \emph{left-inverse}
  matrix $B$, such that $B A' = I_n$, with
  $B = ({A'}^T {A'})^{-1} {A'}^T$.  However, if $r(A') = r < n$, then
  there are $r$ basic variables and $n - r$ free variables. The free
  variables will allow to generate an infinite number of solutions. It
  is possible to select one \emph{particular} solution by assuming
  that all the free variables are zero. The objective function of
  linear programming is then based on the basic variables. However, we
  may also exchange basic and free variables, as is common in linear
  programming, to move to a different particular solution. If
  $r(A') = n$, an LP solution will be the same as the one obtained by
  the linear algebra approach.
\item ($m < n$) The system $A' x = b'$ has \emph{at least one}
  solution if the columns of $A'$ span $\mathbb{R}^m$, i.e.,
  $r(A') = r = m$. There exists an $n \times m$ right-inverse $C$ such
  that $A' C = I_m$, with $C = {A'}^T (A' {A'}^T)^{-1}$. If
  $r(A') < m$, there are $m - r$ constraints on $b'$ for $A' x = b'$
  to be solvable and there are $r$ basic variables and $n - r$ free
  variables.  However, as $A'$ is a surjective linear transformation,
  the solution $x$ obtained in this way is just one of the possible
  solutions (see Examples \ref{ex2} for
  an actual case), and may differ from the LP solution.
\end{itemize}
We give a number of examples to illustrate the theory presented above.
\begin{example}
We adopt the same order of cases distinguished above.
 \begin{itemize}
 \item ($m = n$). For the $m \times n$ matrix $A'$ it holds that $r(A') = m = n$, i.e.,
   the matrix is non-singular and an inverse ${A'}^{-1}$ exists. Here $m = n = 6 = r(A')$,
   and
   \[b' = [0.99\; 0.1\; 1\; 1\; 1\; 1]^T\] 
   \[A' = \left[\begin{array}{cccccc}
 1 & 1 & 0 & 0 & 0 & 0 \\
 1& 0& 0& 0& 0& 0 \\
 1& 1& 1& 0& 0& 0 \\
 1& 0& 0& 1& 0& 0 \\
 0& 1& 0& 0& 1& 0 \\
 0& 0& 1& 0& 0& 1
           \end{array}\right]
       \]
resulting in the LP solution:
\[x = [0.1\;  0.89\; 0.01\; 0.9\;  0.11\; 0.99]^T\]
\[{A'}^{-1} = \left[\begin{array}{rrrrrr}
 0&  1&  0& -0&  0& -0 \\
 1& -1&  0& -0&  0& -0 \\
 -1&  0&  1& -0&  0& -0 \\
 0& -1&  0&  1&  0& -0 \\
 -1&  1&  0&  0&  1& -0 \\
                      1&  0& -1&  0&  0&  1
\end{array}\right]                                         
\]
Clearly, ${A'}^{-1} b' =  x = [0.1\; 0.89\; 0.01\; 0.9\; 0.11\; 0.99]^T$, which is exactly the
same as the LP solution mentioned above.
\item ($m > n$). The $m \times n$ matrix $A'$ has an $n \times m$
  left inverse $B$ as $m > n$. In the example below, $m = 7$, $n = 4 = r(A')$, and
\[b' = [0.1\; 0.9\; 0.1\; 0.9\; 1\; 1\; 1]\]
\[A' = \left[\begin{array}{cccccc}
 1& 0& 0& 0 \\
 0& 1& 0& 0 \\
 1& 0& 0& 0 \\
 0& 1& 0& 0 \\
 1& 1& 0& 0 \\
 1& 0& 1& 0 \\
 0& 1& 0& 1
\end{array}\right]\]
with LP solution equal to $x = [0.1\; 0.9\; 0.9\; 0.1]^T$.  The left inverse $B$
is equal to:
\[B = \left[\begin{array}{rrrrrrr}
  0.375 & -0.125 &  0.375 & -0.125 &  0.25 & 0 & 0\\
 -0.125 &  0.375 & -0.125 &  0.375 &  0.25 & 0 & 0\\
 -0.375 &  0.125 & -0.375 &  0.125 & -0.25 & 1 & 0\\
  0.125 & -0.375 & 0.125  & -0.375 & -0.25 & 0 & 1
\end{array}\right]\]
with $B\, b' = x = [0.1\; 0.9\; 0.9\; 0.1]^T$.
\item ($m < n$). There is only an $n \times m$ right inverse $C$  and $r(A') = m$. Here,
  $ m = 7 = r(A')$ and $n = 8$, with
\[b' = [0.99\; 0.1\; 1\, 1\, 1\, 1\, 1]^T\]
and
 \[A' = \left[\begin{array}{cccccccc}
 0& 1& 1& 1& 0& 0& 0& 0\\
 0& 0& 0& 1& 0& 0& 0& 0\\
 1& 1& 1& 1& 0& 0& 0& 0\\
 1& 0& 0& 0& 1& 0& 0& 0\\
 0& 1& 0& 0& 0& 1& 0& 0\\
 0& 0& 1& 0& 0& 0& 1& 0\\
 0& 0& 0& 1& 0& 0& 0& 1    
\end{array}\right]\]
with LP solution equal to $x = [0.01\;  0.89\;  0\, 0.1\; 0.99\; 0.11\; 1\; 0.9]^T$.
The right inverse $C$ is equal to:
\[C = \left[\begin{array}{rrrrrrr}
 -1&    0&    1&    0&    0&    0&    0\\
 0.5&  -0.5 &  0&    0&    0.25& -0.25&  0\\
 0.5 & -0.5 &  0&    0&   -0.25 & 0.25&  0\\
 0&    1&    0&    0&    0&    0&    0\\
 1&    0&   -1&    1&    0&    0&    0\\
 -0.5&   0.5&   0&    0&    0.75 &  0.25&  0\\
 -0.5 &  0.5 &  0&    0&    0.25&  0.75&  0\\
 0&   -1&    0&    0&    0&    0&    1
 \end{array}\right]
\]
Here it holds that $C\, b' = x'  = [0.01\; 0.445\; 0.445\; 0.1\; 0.99\; 0.555\; 0.555\; 0.9]^T$.
Note that the LP and the linear algebra solution differ. For the LP solution $x$
with $c^T = [10\; 1\; 1\; 1\; 1\; 1\; 1\; 1]$, we have that $c^T x = 3.09$, whereas the linear
algebra solution $x'$ yields $c^T x' = 4.09$, which explains
why it was not selected by LP.
\end{itemize}
\end{example}

\subsection{Permutations of solutions}

In this section, we focus on the submatrices $A$ of $A'$ and vector $u$.
Let $A$ and $A_l$ be two Boolean matrices, where $A_l$ is obtained from
$A$ by pre-multiplying $A$ by a \emph{permutation matrix} $P$, basically
the identity matrix $I$ with rows reordered, with $A_l = P A$; in other
words, $A_l$ is the same as $A$ with rows reordered. For the solutions
of the set of equations $A_l u = b$ we have that
\begin{eqnarray*}
 A_l u =  P A u = b \; \Leftrightarrow \; A u = P^{-1}\,b = P^T\,b = b'
\end{eqnarray*}
(as permutation matrices are orthogonal matrices, they have the
special property that $P^{-1} = P^T$)
which implies that that the solution $u$ does not change. However, the
solution of $A u' = b$ is likely to be different from $A_l u = b$ as for
the row space it holds that ${\cal R}(A^T) = {\cal R}(A_l^T)$, but
$b$ should be in the column space ${\cal R}(A)$, and the column
spaces ${\cal R}(A)$ and ${\cal R}(A_l)$ are likely to differ.

Thus, rather than considering reordering the rows, we should consider
columns of $A$.  We have that the column space ${\cal R}(A)$ remains
the same despite  column reordering, meaning that
$b \in {\cal R}(A)$ if the equation
$A x = b$ has a solution $x$, then also $b \in {\cal R}(A_r)$, with $A_r = A P$, i.e.\
$A$ after post-multiplication by the permutation matrix $P$.
Hence, a starting point for reducing (potentially)
irrelevant permutations of solution vector $u$ is by reducing
the number of reorderings of columns of a previously generated
Boolean matrix $A$.

The following example will demonstrate that a permutation
of a matrix $A$ will not always yield a permutation of a solution.
\begin{example}
We demonstrate the consequences of pre- and post-multiplication by
a permutation matrix $P$.
Let $b = [0.99\; 0.1]^T$ in the equation $A u = b$, with
\[A = \left[\begin{array}{rrr}
     1 & 1 & 0 \\
              0 & 1 & 0
    \end{array}\right]
\]
a (probability vector) solution is $u = [0.89\; 0.1\; 0.01]^T$. Now, let
\[A_l = \left[\begin{array}{rrr}
        0 & 1 & 0 \\
        1 & 1 & 0
       \end{array}\right]
      = P A = \left[\begin{array}{rr}
                0 & 1\\
                1 & 0
              \end{array}\right]
      \left[\begin{array}{rrr}
     1 & 1 & 0 \\
     0 & 1 & 0
    \end{array}\right]\]
be a reordering of $A$; a solution is $u' = [-0.89\; 0.99\; 0.9]^T$,
which is neither a permutation of $u$ nor a probability
vector.
Subsequently, consider
\[A_r = A P' = A \left[\begin{array}{rrr}
                   0 & 1 & 0 \\
                   1 & 0 & 0 \\
                   0 & 0 & 1
                      \end{array}\right]
                   = \left[\begin{array}{rrr}
                1 & 1 & 0 \\
                1 & 0 & 0
               \end{array}\right]\]
where we reordered the columns of matrix $A$. The set of equations
\[A_r u'' = \left[\begin{array}{rrr}
                1 & 1 & 0 \\
                1 & 0 & 0
               \end{array}\right] u'' = b
\]
yield $u'' = [0.1\;  0.89 \; 0.01]^T$, which is a permutation of $u$,
and a probability vector. Another permutation is obtained by
\[A_r' = A P'' = A \left[\begin{array}{rrr}
                   0 & 0 & 1 \\
                   0 & 1 & 0 \\
                   1 & 0 & 0
                      \end{array}\right]
                   = \left[\begin{array}{rrr}
                0 & 1 & 1 \\
                0 & 1 & 0
               \end{array}\right]\]
where again we reordered the columns of matrix $A$. The set of equations
\[A_r' u''' = \left[\begin{array}{rrr}
                0 & 1 & 1 \\
                0 & 1 & 0
               \end{array}\right] u''' = b
\]
yield $u''' = [0.01\;  0.1 \; 0.89]^T$, which is another permutation of $u$,
and a probability vector.    
\end{example}
The example illustrates
the following simple equivalence.
\begin{proposition}
\label{prop.eq}
  Let $u$ be a solution of the set of linear equations $A u = b$, then
  if $AP u' = A_r u' = b$, with solution $u'$ and permutation matrix
  $P$, then vector $u$ is a permutation of vector $u'$ and vice versa.
\end{proposition}
\begin{proof}
  $AP u' = A_r u' = b \Leftrightarrow A (P u') = b \Leftrightarrow A u = b$,
  with $u = Pu'$.
\end{proof}
\noindent
This simple proposition means that rather than first generating
Boolean matrices $A$, one can also generate, more efficiently,
permutations of solutions.  Unfortunately, the proposition does not
prevent redundancy in the generation process of matrices $A$. However,
whereas it is hard to come up with a way to prevent generating
redundant matrices $A$, it is possible, at least partially, to bypass
generating redundant solutions. As discussed in Section
\ref{sec.linalg}, the existence of a solution $u$ of $A u = b$ is
equivalent to the statement that $b \in \mathcal{R}(A)$, and the order
of columns in $A$ does not matter. A simple approach to prevent
repeatedly generating permutations of solution vectors can be obtained
by checking whether a particular column space $\mathcal{R}(A)$ has been
generated before. One can do this by first transforming a Boolean
matrix $A$ to a canonical form with preservation of the column
space. Unfortunately, the actual transformation will cost (although
polynomial) computation time. Alternatively, one can sort the columns
of a matrix $A$ for which a solution of the equation $A u = b$ exists,
and for any newly generated matrix $A'$ check whether the sorted
version of this matrix was previously generated. If it does, then
clearly the solution $u'$ of $A' u' = b$, which is guaranteed to
exist, is a permutation of $u$.
\begin{proposition}
  \label{prop.lex}
  Let $A u = b$ be a system of linear equations with at least
  one solution vector $u$, and let $\textrm{sort}(A) = A'$ be a matrix
  with the columns of $A$ in lexicographic order. Then, $A' u' = b$
  has a solution $u'$, with $u' = P u$, i.e.\ $u'$ is a permutation
  of $u$ and vice versa, where $P$ is a permutation matrix.
\end{proposition}
\begin{proof}
Simply note that $A' P u = A u = b$.
\end{proof}
Clearly, Proposition~(\ref{prop.lex}) is almost identical to Proposition~(\ref{prop.eq})
except that Proposition~(\ref{prop.lex}) can be used to filter out matrices,
once generated, from further processing.

\section{A heuristic search algorithm}
\label{sec.alg}

As suggested in the previous section: there is not always a unique
solution to the equation $A' x = b'$, with an unknown Boolean submatrix
$A$ of $A'$, an unknown vector $x = [u \; w]^T$, with $u$ a probability vector
and slack variables $w$, and
the \emph{known} vector $b'$. The only thing we know beforehand,
namely that $A$ is a Boolean (binary) matrix, yields as a possible approach
to systematically generate the matrix $A$, and subsequently
solve for $[u\; w]^T$ using linear programming. This is what the
python code below does, where we have recoded $[u\; w]^T$ simply
as $u$ (and $u.x$ is the feasible linear programming solution).

In addition, in the code below, we use the linear algebra
approach to the existence and uniqueness of solutions
as an alternative to the LP approach.  This will
allow exploring the alternative cases, based on the dimensions
of the matrix $A$, introduced in Section \ref{sec.exist}.

As a simple heuristic we use a cache to store column sorted matrices
$A$ for which a solution was found using LP as a partially successful
method to prevent generating the same permutated solutions multiple
times. As mentioned above, a better method would check for equivalence
of the column space of the matrix, but that would be even more time consuming.
Finally, it is, in principle, also possible that for two different matrices $A, A'$
it holds that  $A u = A' u = b$, i.e.\ the same solution $u$
is obtained. We give an example below.
\begin{example}
Consider the following two examples:
\begin{eqnarray*}
  A &= &\left[\begin{array}{rrrrr}
  1 & 0 & 0 & 0 & 0 \\ 
  1 & 1 & 1 & 1 & 0 \\ 
  0 & 1 & 0 & 0 & 0 \\ 
  0 & 1 & 1 & 0 & 0 \\ 
  0 & 0 & 1 & 0 & 0 \\ 
  1 & 1 & 1 & 1 & 1
  \end{array}\right] \\[2ex]
u & = & \left[0.1\;  0.4\;  0.3\;  0.19\; 0.01\right]^T \\[2ex]
  A' &=& \left[\begin{array}{rrrrr}
 1 & 0 & 0 & 0 & 0 \\
 1 & 1 & 1 & 1 & 0 \\ 
 0 & 1 & 0 & 0 & 0 \\ 
 0 & 1 & 1 & 0 & 0 \\ 
 1 & 0 & 0 & 1 & 1 \\ 
 1 & 1 & 1 & 1 & 1 
\end{array}\right]
\\[2ex]
u' &= & \left[0.1 \; 0.4 \; 0.3 \; 0.19 \; 0.01\right]^T
\end{eqnarray*}
The two matrices $A, A'$ differ, and yet the solution obtained
is identical
\end{example}
The python code presented below, that was developed
according to the theory elaborated in the section above,
supports experimentation of the process of
transforming a BN to a PSCM.
\begin{verbatim}
import numpy as np
from scipy import linalg, optimize
import copy

def Ext(setofsets, add):
    """
    add new elements to each set of powerset
    """
    temp = copy.deepcopy(setofsets)
    l = len(setofsets)
    for i in range(l):
        setofsets[i].append(add) # setofsets is mutated
    temp.extend(setofsets)
    return temp
    
def Powerset(set):
    """
    compute powerset of set
    """
    if len(set) != 0:
        car = set.pop(-1)
        return Ext(Powerset(set), car)
    else:
        return [[]]

def ComputeSolutions(m, n, b, limit):
    A = np.zeros((m+1, n))
    indices = list(range(0, n))
    ps = Powerset(indices)
    pslist = list(range(0, m))
    for i in range(0, m):
        pslist[i] = copy.deepcopy(ps)
        pslist[i].pop(0) # remove empty list
    O = np.zeros((m+1, n)) # not used, just for consistency with theory
    I = np.eye(n)
    e = np.ones(n)
    A[m, 0:n] = e
    o = np.zeros(n) # not used, just for consistency with theory
    Ap = np.zeros((m + n + 1, 2 * n))
    Ap[0:m+1, 0:n] = A 
    Ap[m+1:(n + m + 1), 0:n] = I
    Ap[m+1:(n + m + 1), n:(2 * n)] = I
    Solve(Ap, A, m, n, pslist, [], limit)

def Occurs(A, arraylist):
    """
    Array A is member of a list of arrays
    """
    if arraylist == []:
        return(False)
    elif np.all(A == arraylist.pop()):
        return(True)
    else:
        return(Occurs(A, arraylist))

def Solve(Ap, A, m, n, sos, cache, limit):
    """
    vary in indices recursively going through set of sets (sos);
    limit is largest size of list that is used, to reduce the search
    under the assumption (heuristic) that smaller index sets are more
    likely to provide a solution.
    """
    for set in sos[0]:
        if len(set) <= limit:
            A[m - len(sos), set] = 1
            sortedarray = np.sort(A, 1)
            if len(sos) > 1:
                Solve(Ap, A, m, n, sos[1:len(sos)], cache, limit)
            Ap[0:m + 1, 0:n] = A
            obj = np.ones(2 * n)
            
            if not(Occurs(sortedarray, cache.copy())):
                u = optimize.linprog(obj, A_eq = Ap, b_eq = b,
                                     bounds = (0, None))
                if u.success:
                    print("Successful matrix & solution:")
                    print(Ap, "\n", u.x) 
                    if (Ap.shape[0] == Ap.shape[1]):
                        det = linalg.det(Ap)
                        if det == 0:
                           print("\nSingular square matrix \n")
                        else:
                            Ainv = linalg.inv(Ap)
                            u2 = np.dot(Ainv, b)
                            print("\nInverse * b = ", u2, "\n")
                    elif (Ap.shape[0] > Ap.shape[1]):
                        leftinverse = np.dot(linalg.inv(np.dot(np.transpose(Ap),Ap)),
                                             np.transpose(Ap))
                        u2 = np.dot(leftinverse, b)
                        print("\nLeftinverse * b = ", u2, "\n")
                    else:
                        rightinverse = np.dot(np.transpose(Ap),
                                              linalg.inv(np.dot(Ap,np.transpose(Ap))))
                        u2 = np.dot(rightinverse, b)
                        print("\nRight inverse * b = ", u2, "\n")
                    cache.append(sortedarray)
            A[m - len(sos), set] = 0
\end{verbatim}
Next, an example run of the python code
with some solutions generated for Example \ref{ex1} is shown:
\begin{verbatim}
m = 2; n = 3; limit = 2 # Example 1

b = np.array([[0.99], [0.1], [1], [1], [1], [1]])

>>> ComputeSolutions(m, n, b, lim)

Successful matrix & solution:
[[1. 1. 0. 0. 0. 0.]
 [1. 0. 0. 0. 0. 0.]
 [1. 1. 1. 0. 0. 0.]
 [1. 0. 0. 1. 0. 0.]
 [0. 1. 0. 0. 1. 0.]
 [0. 0. 1. 0. 0. 1.]]

[0.1  0.89 0.01 0.9  0.11 0.99]

Inverse * b =  [[0.1]
 [0.89]
 [0.01]
 [0.9 ]
 [0.11]
 [0.99]] 

Successful matrix & solution:
[[1. 1. 0. 0. 0. 0.]
 [1. 0. 1. 0. 0. 0.]
 [1. 1. 1. 0. 0. 0.]
 [1. 0. 0. 1. 0. 0.]
 [0. 1. 0. 0. 1. 0.]
 [0. 0. 1. 0. 0. 1.]]

[0.09 0.9  0.01 0.91 0.1  0.99]

Inverse * b =  [[0.09]
 [0.9 ]
 [0.01]
 [0.91]
 [0.1 ]
 [0.99]] 
\end{verbatim}

\section{Discussion}
\label{sec.disc}

In this paper, we introduced a systematic way for transforming a BN to
a PSCM. As there is already a significant amount of theory, including
its implementations within software environments, to learn both graph
structure and parameterisations of a BN from data, the idea of
transforming a BN to a PSCM is appealing. Practically speaking, one
may start by developing a BN from a mixture of data and expert
knowledge, which, after having been successfully evaluated, may be
turned into a PSCM. A requirement for a successful transformation is
that a PSCM's exogenous variables are introduced with a domain consisting
of sufficient number of values to preserve the probabilistic
information. In addition, the causal functions or deterministic
conditional probability distributions that admit generating an
equivalence PSCM are usually not unique. Actually, the process of
generating solutions has combinatorial properties, and effective
heuristics to reduce the number of alternative solutions generated
have not been found so far.

One disadvantage of a PSCM is its unclear semantics as different
definitions of probability distributions of the exogenous variables
may give equivalent results. This also raises doubts about the
usefulness of PSCMs as a suitable alternative to BNs, certainly in the
context of machine learning.

\bibliographystyle{plain}
\bibliography{reference}
\end{document}